\documentclass[12pt,journal,cspaper,compsoc, onecolumn]{IEEEtran}
\usepackage{algorithmic}
\usepackage{algorithm, balance}
\usepackage{times}
\usepackage{hyperref}
\usepackage{epsfig, graphicx}
\usepackage{amsmath,amssymb}
\usepackage{amssymb,MnSymbol,wasysym}
\usepackage{stfloats}
\usepackage{multirow,booktabs,supertabular}
\usepackage{dsfont}

\newtheorem{theorem}{Theorem}[section]
\newtheorem{lemma}[theorem]{Lemma}

\newenvironment{proof}[1][Proof]{\begin{trivlist}
\item[\hskip \labelsep {\bfseries #1}]}{\end{trivlist}}
\newenvironment{definition}[1][Definition]{\begin{trivlist}
\item[\hskip \labelsep {\bfseries #1}]}{\end{trivlist}}

\newcommand{\defeq}{\mathrel{\mathop:}=}

\begin{document}

\title{
Semi-supervised Spectral Clustering \\for Classification
}
\author{Arif~Mahmood and Ajmal~S.~Mian
\IEEEcompsocitemizethanks{\IEEEcompsocthanksitem Both authors are in the School of Computer Science and Software Engineering, The University of
Western Australia, Crawley WA 6009 {emails: \{arif.mahmood, ajmal.mian\}@uwa.edu.au} }}
\markboth{December 2013}{Arif \MakeLowercase{\textit{et al.}}: Semi-supervised Clustering  }

\IEEEcompsoctitleabstractindextext{\begin{abstract}
We propose a Classification Via Clustering (CVC) algorithm which enables existing clustering methods to be efficiently employed in classification problems. In CVC, training and test data are co-clustered and class-cluster distributions are used to find the label of the test data. To determine an efficient number of clusters, a Semi-supervised Hierarchical Clustering (SHC) algorithm is proposed. Clusters are obtained by hierarchically applying two-way NCut by using signs of the Fiedler vector of the normalized graph Laplacian. To this end, a Direct Fiedler Vector Computation algorithm is proposed. The graph cut is based on the data structure and does not consider labels. Labels are used only to define the stopping criterion for graph cut.  We propose clustering to be performed on the Grassmannian manifolds facilitating the formation of spectral ensembles. The proposed algorithm outperformed state-of-the-art image-set classification algorithms on five standard datasets.
  
\end{abstract}}
\IEEEdisplaynotcompsoctitleabstractindextext

\maketitle

\IEEEpeerreviewmaketitle
\section*{Currently this article is under review and will be redirected to the organizational web page once accepted}
\vspace{10mm}

\section{Introduction}
\label{sec:intro}

Clustering finds the intrinsic data structure by splitting the data into similar clusters whereas classification assigns labels based on prior knowledge. Thus, clustering follows the intrinsic data boundaries whereas classification follows externally imposed boundaries. The two boundaries are generally different resulting in clusters across multiple classes (Fig.~\ref{fig:pcaPlot}). Due to this difficulty, clustering has not been widely used for classification. We bridge this gap by proposing a classification algorithm based on semi-supervised clustering of labelled data combined with unlabelled data. The proposed approach is very powerful in the context of image-set classification~\cite{chenimproved,  chen2012dictionary, ImagesetAlignment2012,harandi2011graph,Hu_TPAMI_2012,ProbElasticMatching,Nishiyama_CVPR_2007, wang2012manifold, VidClust2013}.

\begin{figure}[t]
\centering
\includegraphics[width=8.4cm]{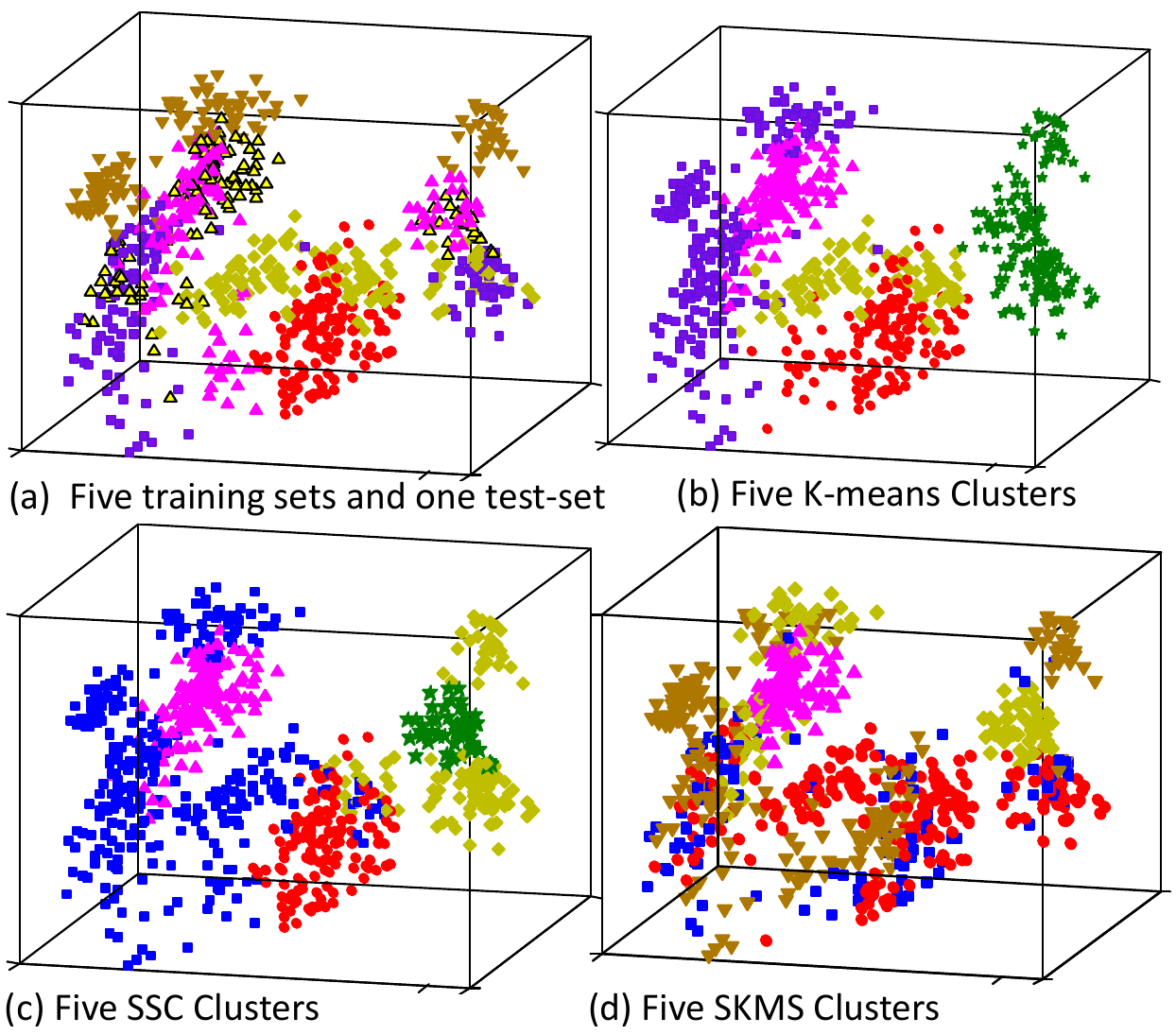} 
\caption{Plot of top three eigen-coefficients of 6 facial image-sets of 5 subjects from the CMU Mobo data~\cite{Gross_TR_2001}. (a) Ground-truth identities labeled with different shapes and colors. Test set is yellow triangles.  Five clusters obtained with (b) K-means, (c) SSC~\cite{ElhamPAMI} and (d) SKMS~\cite{Anand}. In each case, clusters are across class boundaries. No class cluster correspondence exists. }
\label{fig:pcaPlot}
\end{figure}
\begin{figure*}[t]
\centering
\includegraphics[width=17.5cm] {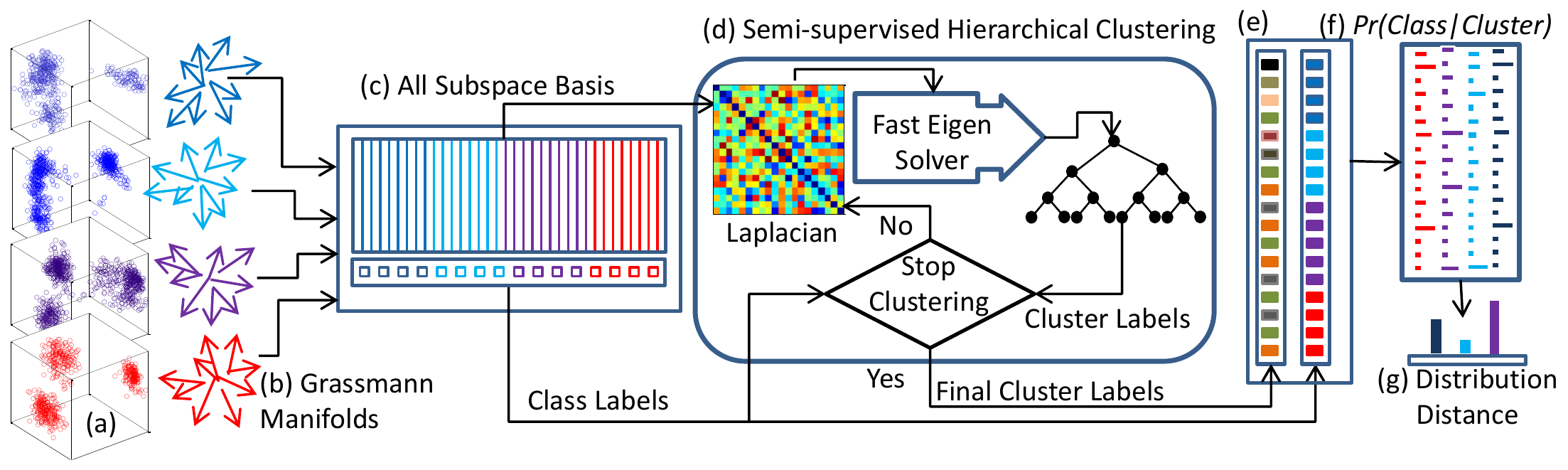}
\caption{Proposed algorithm:(a) Face manifolds of 4 subjects from CMU Mobo dataset in PCA space. (b) Each set transformed to a Grassmannian manifold. (c) Data matrix and class label vector. (d) Semi-supervised Hierarchical Clustering Algorithm: normalized graph laplacian, a novel eigensolver and supervised termination criterion. (e) Each class element assigned a  cluster label. (f) Probability distribution of each class over the set of clusters. (g) Distribution based distance measure.}
\label{fig:blockDiagram}
\end{figure*}

Unsupervised  data clustering has been extensively studied for the analysis of complex data \cite{chen2009spectral,ElhamPAMI,Fitzgibbon_CVPR_2003,fowlkes2004spectral,Kwok,von2007tutorial}. Despite significant advances, clustering and classification have remained two separate streams with no direct mapping. Figure \ref{fig:pcaPlot} compares the ground truth labels with K-means clustering, Sparse Subspace Clustering (SSC)~\cite{ElhamPAMI} and Semi-supervised Kernel Mean Shift (SKMS) clustering~\cite{Anand}. Some of the basic questions that still remain unanswered include finding the appropriate number of clusters and establishing class-cluster correspondence when the clusters overlap with multiple classes. One can observe in Figure \ref{fig:pcaPlot} that no well defined class-cluster correspondence exists  and an optimal number of clusters cannot be defined in general. Even the latest algorithms such as Semi-supervised Kernel Mean Shift clustering \cite{Anand} does not make the problem easier. Both these ill-posed questions have long occluded the utilization of clustering in classification tasks. One of the main contributions of this paper is bridging the gap between clustering and classification. The proposed algorithm does not require a unique class-cluster mapping, and can efficiently handle clusters across class boundaries as well as multiple clusters within the same class.

Our first contribution is a Classification Via Clustering (CVC) algorithm where the final classification decision is based on semi-supervised clusters computed over all data. This ensures global classification decisions as opposed to the local decisions made in the existing classification literature~\cite{harandi2011graph,Hu_TPAMI_2012,ProbElasticMatching,Nishiyama_CVPR_2007, wang2012manifold, VidClust2013}. For this purpose, we apply semi-supervised clustering on the combined training and test data without enforcing class boundaries.  We compute the clusters based on the data characteristics without using label information. The probability distribution of each class over the set of  clusters is computed using the label information. This distribution can be thought of as a compact representation of the class. Classification is performed by measuring  distances  between the probability distribution of the test data from each training class. The proposed CVC algorithm is generic and applicable to any  clustering algorithm.

Our second contribution is a Semi-supervised Hierarchical Clustering (SHC) algorithm where every parent cluster is partitioned into  two child clusters in an unsupervised way while the labels are used only as a stopping criterion for partitioning. Note that the term semi-supervised is used in a different context compared to the existing semi-supervised clustering algorithms, such as SKMS~\cite{Anand}, where the label information forms a part of the main objective function to be minimized for clustering. In our case, the objective function is independent of the class labels. Although, SHC can be used with any partitioning criterion, we consider NCut graph partitioning where the two partitions are determined by the signs of the Fiedler vector of the Laplacian matrix. Our motiviation for this choice comes from the recent advances in spectral clustering~\cite{chen2009spectral,ElhamPAMI,fowlkes2004spectral,Kwok,ngspectral}.

Our third contribution is a Direct Fiedler Vector Computation (DFVC) algorithm which is based on the shifted inverse iteration method. Hierarchical spectral clustering was not previously considered a viable option because all the eigenvectors were required to be computed while only the second least significant (Fiedler Vector) was used. The proposed DFVC algorithm solves this problem. Moreover, existing spectral clustering research is based on off-the-shelf eigen-solvers which aim to accurately find the magnitudes of the eigenvector coefficients even though only their signs are required for partitioning. The convergence of our DFVC algorithm is based on the signs.

We apply the proposed algorithms to the problem of image-set classification where the training classes and the test class consist of a collection of images. The proposed algorithms can be directly applied to the raw data however, we represent each image-set with a Grassmannian manifold and perform clustering on the manifold basis. This strategy reduces computational complexity and increases discrimination and robustness to noise. The use of Grassmann manifolds enables us to make an ensemble of spectral classifiers, each based on a different dimensionality of the manifold. This significantly increases the accuracy and also gives reliability (confidence) of the label assignment. An overview of our approach is given in Fig.~\ref{fig:blockDiagram}. Our initial work in this direction appeared in \cite{MyCVPR14} where we computed all the eigenvectors and partitioned the data into two or more clusters.  In this paper, we perform only binary partitioning and propose a Direct Fiedler Vector Computation (DFVC) algorithm thus avoiding the computation of all eigenvectors.

Experiments were performed on three standard face image-sets (Honda~\cite{Lee_CVPR_2003}, CMU Mobo~\cite{Gross_TR_2001} and YouTube Celebrities~\cite{Kim_CVPR_2008}), an object categorization (ETH 80)~\cite{ETH80}, and Cambridge hand gesture~\cite{Camdatabase} datasets. Results were compared to seven state of the art algorithms. The proposed technique achieved higher accuracy on all datasets. The maximum improvement was observed on the most challenging You-Tube dataset where our algorithm achieved 11.03\% higher accuracy than the previous best reported.  

\begin{algorithm}[t]
\caption{\textbf{CVC:} Classification Via Clustering }
\begin{algorithmic}
 \REQUIRE {\small Gallery Sets:} $\mathcal{G}$,  {\small Probe Set:} $c_p$, {\small Gallery Labels:} $\ell_\mathcal{G}$
 \ENSURE  ${\ell_p}$ ~~~~~~~~~~~~~~~~~~~~~~~~~~~~~~~~~~~{\small \{Probe Set Label\}}
 \STATE $\mathcal{D}=[ \mathcal{G}~~c_p]$
 \STATE $n_c \leftarrow$ unique($\ell_\mathcal{G}$)~~~~~~~~~~~~~~~~~~~~~~~ {\small \{Number of Classes\}}
 \STATE $\ell_{\mathcal{D}}\leftarrow [\ell_\mathcal{G}~~ \widehat{\ell}_{p}]$,~~~~~~~~~~ {\small \{$\widehat{\ell}_{p}=n_c+1$ is dummy probe label\}}
 \STATE  $[\ell_k, n_k] \leftarrow$ \textbf{SHC} $(\mathcal{D},\ell_{\mathcal{D}})$~~~~~~~~~~{\small  \{Semi-sup. Hierar. Clust.\}}
 \STATE  $H_{cc} \leftarrow $zeros$(n_c+1,n_k)$ ~~~~~~~~~~{\small \{Class cluster histogram\}}  
 \FORALL {$i\leftarrow 1:$ length($\ell_{\mathcal{D}}$)} 
 \STATE $H_{cc}(\ell_\mathcal{D}(i),\ell_k(i))=H_{cc}(\ell_\mathcal{D}(i),\ell_k(i))+1$
 \ENDFOR
 \FORALL {$i\leftarrow 1:n_c+1$}
 \STATE $H_{cc}(i,:)=H_{cc}(i,:)/\sum_{j=1}^{n_k}{H_{cc}(i,j)}$
 \ENDFOR
 \STATE $\textbf{p}_p\leftarrow H_{cc}(n_c+1,:)$~~~~~~~~~~~~~~~~~~~~{\small \{Probe cluster distribution\}}
 \FORALL {$i\leftarrow 1:n_c$}
  \STATE  $\textbf{p}_i\leftarrow H_{cc}(i,:)$~~~~~~~~~~~~~~~~~~~~{\small \{Class cluster distribution\}}
  \STATE $\mathcal{B}(i)=-\ln\sum\limits_{j=1}^{n_k}\sqrt{p_{i}(j)p_{p}(j)}$ ~~{\small\{Bhattacharyya distance\}}
  \ENDFOR
 \STATE $\ell_p \leftarrow \ell_\mathcal{G}(i_{\min})$ such that $i_{\min}\equiv \min_i (\mathcal{B})$
 \end{algorithmic}
\label{algo:CVC}
\end{algorithm}

\section{Related Work}

Many existing image-set classification techniques are variants of the  nearest neighbor (NN) algorithm where the NN distance is measured under some constraint such as representing sets with affine or convex hulls~\cite{Cevikalp_CVPR_2010}, regularized affine hull~\cite{RegNearestPoints},  or using the sparsity constraint to find the nearest points between image-sets~\cite{Hu_TPAMI_2012}. Since NN techniques utilize only a small part of the available data,  they  are more vulnerable to outliers. 

At the other end of the spectrum are algorithms that represent the holistic set structure, generally as a linear subspace, and compute similarity as canonical correlations or principle angles~\cite{Kim_TPAMI_2007}.  However, the global structure may be  a non-linear complex manifold and representing it with a single subspace may lead to incorrect classification~\cite{chenimproved}. Discriminant analysis has been used to force the class boundaries by finding a space where each class is more compact while different classes are apart. Due to multi-modal nature of the sets, such an optimization may not scale the inter class distances appropriately (Fig. \ref{fig:pcaPlot}a). In the middle of the spectrum are algorithms that divide an image-set into  multiple local clusters (local subspaces or manifolds) and measure cluster-to-cluster distance~\cite{chenimproved,Wang_CVPR_2009,wang2012manifold}. In all these techniques, label assignment decision is dominated  by either a few data points, only one cluster, one local subspace, or one basis of the global set structure,  while the rest of the  set data or local structure variations are ignored. In contrast, our approach uses all the samples and exploits the local and global structure of the image-sets.

\section{Classification Via Clustering}
Classification algorithms enforce the class boundaries in a supervised way which may not be optimal because training data label assignment  is often manual and based on the real world semantics or prior information instead of the underlying data characteristics. For example, all training images of the same person are pre-assigned the same label despite  that the intra-person distance may exceed  inter-person distance. Therefore, the intrinsic data clusters will not align  with the imposed class boundaries.

We propose to compute the class to class distance based on the probability distribution of each class over a finite set of  clusters. In the proposed technique,  clustering is performed on all data comprising the training and test sets combined without considering their labels. By doing so, we ensure natural clusters based on the inherent data structure. Clusters are allowed to be formed across two or more training classes. Once an appropriate number of clusters is obtained, we use the labels of training data to compute the probability distribution of each class over the set of clusters. Then we use distribution based distance measures to find which class is the nearest to the test set. Algorithm \ref{algo:CVC} shows the basic steps of the proposed approach.

Let $\mathcal{G}=\{X_i\}_{i=1}^{g} \in \mathds{R}^{l \times n_g}$ be the gallery containing labeled training data, where $n_g=\sum_{i=1}^{g}{n_i}$ are the total number of data points in the gallery and  $n_i$ be the data points in a set $X_i=\{x_j\}_{j=1}^{n_i} \in\mathds{R}^{l\times n_i}$. A data point $x_j\in \mathds{R}^l$ could be a feature vector (e.g.~LBP or HoG) or simply the pixel values. The number of classes in the gallery $n_c$ are generally less than or equal to the number of sets $ n_c \le g$. Let $c_p=\{x_i\}_{i=1}^{n_p} \in \mathds{R}^{l \times n_p}$ be the probe-set with a dummy label $n_c+1$. We make a data matrix by appending all gallery sets and the probe set: $\mathcal{D}=[\mathcal{G}~~c_p]\in \mathds{R}^{l \times n_d}$, where $n_d=n_g+n_p$. Let $\ell_\mathcal{G}\in \mathds{R}^{ n_g}$ be the labels of the gallery data and $\widehat{\ell}_{k} \in \mathds{R}^{n_p}$ be the dummy labels of the probe set,  selected as $n_c+1$. A  label matrix is formed as $\ell_\mathcal{D}=[\ell_\mathcal{G}~~\widehat{\ell}_{k}]\in \mathds{R}^{ n_d}$.

For the purpose of clustering we propose  Semi-supervised Hierarchical Clustering (SHC) due to better control on the quality of the clusters as discussed in Section \ref{SHC}. The output of the SHC is a cluster label array $\ell_k$ and the number of clusters $n_k$. Note that for the case of existing unsupervised clustering algorithms such as SSC or K-means,  $n_k$ is a user defined parameter while our   proposed SHC  Algorithm  automatically finds  $n_k$. Using the label arrays $\ell_\mathcal{D}$ and  $\ell_k$ we compute a 2D class-cluster histogram $H_{cc}$. Each row of $H_{cc}$ corresponds to a specific class and each column corresponds to a specific cluster. For a class $c_i$  let $\textbf{p}_i=H_{cc}(i,:)/\sum_{k=1}^{n_k}{H_{cc}(i,k)} \in \mathds{R}^{n_k}$ be the distribution over all clusters, $\sum_{k=1}^{n_k}{\textbf{p}_i[k]}=1$ and $1 \ge \textbf{p}_i[k] \ge 0$. Since we  set dummy probe label to be $n_c+1$,  the last row of $H_{cc}$ is the probe set distribution over all clusters $\textbf{p}_p=H_{cc}(n_c+1,:)/\sum_{k=1}^{n_k}{H_{cc}(i,k)}$.  

Distance between the two distributions $\textbf{p}_i$ and $\textbf{p}_p$ can be found by using an appropriate distance measure such as  Bhattacharyya~\cite{div1967} distance  
\begin{equation}\label{bhat}
\mathcal{B}_{i,p}=-\ln \sum\limits_{k=1}^{n_k}\sqrt{\textbf{p}_{i}(k)\textbf{p}_{p}(k)}.
\end{equation}
Bhattacharyya distance is based on the angle between the square-root of the two distributions: 
$\mathcal{B}_{i,p}=-\ln (\cos(\theta_{ip}))$.  Another closely related measure is
Hellinger distance~\cite{beran1977minimum} which is the $\ell_2$ norm of the difference between the  square-root of the two distributions \begin{equation}\label{Hal}
\mathcal{H}_{i,p}= \frac{1}{\sqrt{2}} \sqrt{\sum_{k=1}^{n_k}{(\sqrt{\textbf{{p}}_i[k]}-\sqrt{{\textbf{p}}_p[k]})^2}}.
\end{equation}

Fig.~\ref{fig:Fig1Stat} shows the class cluster distributions for the three clustering results shown in Fig.~\ref{fig:pcaPlot}.  Distance  $\mathcal{B}_{i,p}$ for each algorithm is shown in Fig.~\ref{fig:Fig1Stat}d. Minimum value of $\mathcal{B}_{i,p}$ is found to be  \{0.0680, 0.00423, 0.0342\} for k-means, SSC and SKMS respectively. In all three cases, the label of the test class is correctly found. We consider  the ratio of rank 2 to rank 1 distance  as SNR. A larger value of SNR shows more robustness to noise. In this experiment SNR of k-means, SSC and SKMS is  \{2.63, 78.84, 5.20\} respectively indicating that SSC based classification is more robust compared to the others.

\begin{figure}[t]
\centering
\includegraphics[width=8.5cm]{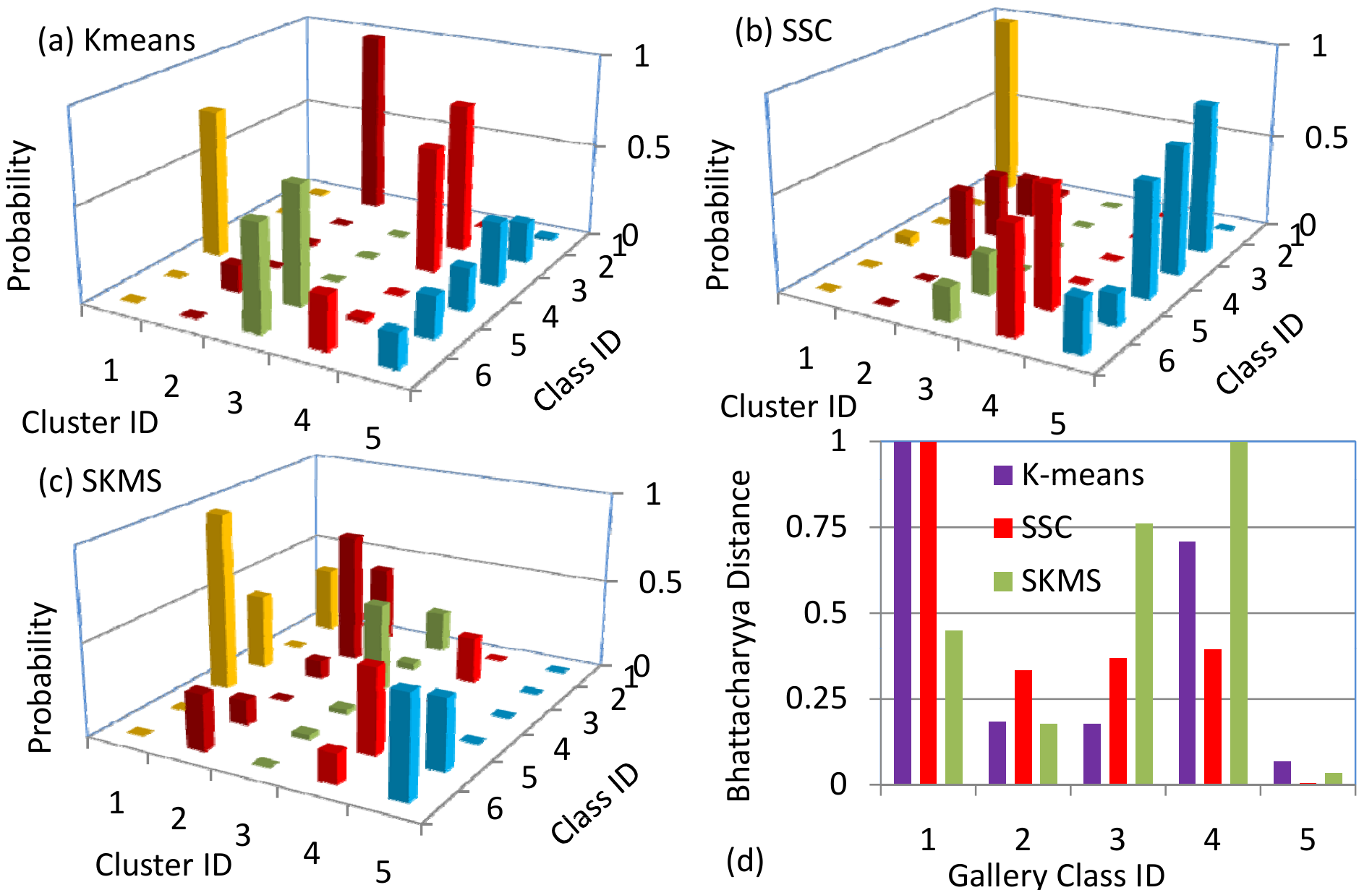}
\caption{(a-c) Plots of class-cluster distributions ($H_{cc}$) of the five gallery classes and 6-th test class over five clusters obtained by k-means, SSC and SKMS shown in Fig. \ref{fig:pcaPlot} (d) Bhattacharyya distance  of the test set from each gallery class.}
\label{fig:Fig1Stat}
\end{figure}
\section{Optimizing the Number of Clusters }
Algorithm \ref{algo:CVC} does not impose any constraint on the number of clusters however,  in Lemma \ref{lam2.1} we argue that an optimal number of clusters exists and can be found in the context of classification problem. The derivation of Lemma \ref{lam2.1} is based on the notion of {\em conditional orthogonality} and {\em indivisibility} of clusters as defined below. We assume that classes $c_i$ and $c_j$ belong to the gallery ${\mathcal{G}}$  while $c_p$ is the probe set with unknown label.

\begin{definition}[Conditional Orthogonality] Class cluster distribution of class $c_i$ is {\em conditionally orthogonal} to the distribution of a class $c_j$ w.r.t the distribution of probe set $c_p$ if \begin{equation}\label{CO}
(\textbf{p}_i \perp_{c} \textbf{p}_j)_{\textbf{p}_p} \defeq \llangle(\textbf{p}_i \land \textbf{p}_p), (\textbf{p}_j \land \textbf{p}_p)\rrangle=0 ~~\forall (c_i,c_j)\in \mathcal{G}.
\end{equation} 
\end{definition}
\begin{definition}[Indivisible Cluster] A cluster $k$ is indivisible from the probe set label estimation perspective  if $\forall (c_i,c_j)\in \mathcal{G}$ $$\label{indC}
p_i[k]\land p_j[k]\land p_p[k]=0. 
$$  $$p_i[k] \lor p_j[k] \lor p_p[k]=1.$$   A cluster is {\em indivisible} if at least one of the three probabilities, $p_i[k]$, $p_j[k]$ and $p_p[k]$ is zero. If all clusters become  indivisible, class cluster distributions of all gallery classes will become {\em conditionally orthogonal}. 
\end{definition}
\begin{lemma}
The optimal number of  clusters for a set labeling problem are the minimum number of clusters such that all  class-cluster distributions become {\em conditionally orthogonal}. \begin{equation}\label{lam1}
n_k^* \triangleq \min_{n_k} (\textbf{p}_i \perp_{c} \textbf{p}_j)_{\textbf{p}_p}~~~~~ \forall c_i,c_j\in \mathcal{G} 
\end{equation}\label{lam2.1}
\end{lemma}
\begin{proof}
Consider an indivisible cluster with two non-zero probabilities $p_i$ and $p_p$ corresponding to the gallery class $c_i$ and the probe set. Suppose this cluster is divided into $n\ge 2$ child clusters with $q_i[1],q_i[2]\cdots q_i[n]$ and $q_p[1],q_p[2]\cdots q_p[n]$ probabilities such that     $q_i[1]+q_i[2]\cdots +q_i[n]=p_i$ and $q_p[1]+q_p[2]\cdots +q_p[n]=p_p$. Then $$p_ip_p\ge \sum_{j=1}^n q_i[j]q_p[j].$$
Therefore, reduction in the cluster size will cause increase in the distances  
 $\mathcal{B}_{i,p}$ and $\mathcal{H}_{i,p}$. Moreover,  some of the probabilities $q_i[j]$ and $q_p[j]$ may become zero causing the loss of useful discriminating information.

Division of the clusters with $p_p=0$ will not have any effect on  $\mathcal{B}_{i,p}$ while $\mathcal{H}_{i,p}$ will decrease.  

Decreasing the number of clusters such that $n_k < n_k^*$, will result in some clusters with overlapped class-cluster distributions leading to an increased intra-class similarity and reduced discrimination. Therefore, $n_k^*$ are the optimal number of clusters for classification. 
\end{proof}

Existing unsupervised clustering algorithms do not ensure the clusters to be indivisible and require the number of clusters to be user defined. In the next section, we propose  Semi-supervised Hierarchical  Clustering (SHC) which efficiently solves this problem.   

\begin{algorithm}[t]
\caption{\textbf{SHC} Semi-supervised Hierarchical Clustering}
\begin{algorithmic}
 \REQUIRE {\small Global proximity matrix:} $A$, {\small class labels array:} $\ell_\mathcal{D}$, {\small cluster label matrix:} $\ell_k$, {\small current cluster ID:} $c$, {\small current recursion level:} $r$, {\small dummy probe label:} $\widehat{\ell}_p$  
 \ENSURE  $\ell'_k$ ~~~~~~~~~~~~~~~~~~~~~{\small \{Updated cluster labels matrix\}}
 \STATE  $\ell_b\leftarrow \ell_k(r,:)$~~~~~~~~~~~~~~~~~~~ {\small \{Cluster labels at current level\}}
 \STATE  $i_c \leftarrow \ell_b==c$ ~~~~~~~~~~~~~~~~~{\small \{Current cluster indicator vector\}}
 \STATE  $\ell_c \leftarrow \ell_\mathcal{D}(i_c)$ ~~~~~~~~~~~~~~{\small \{Class labels in the current cluster\}}
 \IF {Indivisible($\ell_c,\widehat{\ell}_p $)} 
 \STATE\textbf{ Return} ~~~~~~~~~~~~~~~~~~~~~~~~~~{\small \{Retain indivisible cluster\}}
 \ENDIF
 \STATE  $A_c \leftarrow A(i_c,i_c)$ ~~~~~~~~~~~~~~~ {\small \{Local proximity matrix: $n \times n$\}}
 \FORALL {${j=1:n}$}
 \STATE $D_c(j,j)=\sum_{i=1}^{n}{A_c(i,j)}$ ~~~~~~~~~~~~~{\small \{Local degree matrix\}}
 \ENDFOR
 \STATE $L_s \leftarrow {D_c^{-1/2}}(D_c-A_c){D_c^{-1/2}}$ ~~~~~~~~~~~~~\{\small {Laplacian Matrix\}}
 \STATE $u_{n-1}\leftarrow$\textbf{DFVC}$(Ls, D_c)$~~~~~~~~ {\small \{Direct Fiedler Vector Comp.\}}
 \STATE $\hat{\ell}=u_{n-1} \ge 0$~~~~~~~~~~~~~~~~~~~~~~~~~~~ {\small \{Sign based partitioning\}}
  \FORALL {$p=0:1$}
  \STATE $c \leftarrow c+1$
  \STATE $\ell'_k \leftarrow$ Update-Cluster-Labels $(\ell_k,\hat{\ell},c,p,r)$
  \STATE $\ell'_k \leftarrow\textbf{ SHC}(A,\ell_\mathcal{D},\ell'_k,c,r+1,t)$ ~~~~~~~~~~~{\small \{Recursive Call\}} 
  \ENDFOR
 \end{algorithmic}
\label{algo:SHC}
\end{algorithm}

\section{SHC: Semi-supervised Hierarchical Clustering } \label{SHC}
The class cluster distributions shown in Fig.~\ref{fig:Fig1Stat} are not conditionally orthogonal because two clusters need further partitioning. 
However, if the number of clusters is blindly increased, other   indivisible clusters  may  get partitioned.  Such a break down of indivisible clusters will reduce the discrimination capability of the CVC algorithm. To this end, we propose a semi-supervised  algorithm based on  two-way hierarchical clustering which can identify indivisible clusters and hence avoid further partitioning of such clusters. A parent cluster is partitioned into two child clusters only if it is divisible.  Once all clusters become {\em indivisible}, the algorithm stops and hence the optimal number of clusters is automatically determined.

 Algorithm \ref{algo:SHC} recursively implements the proposed Semi-supervised Hierarchical Clustering (SHC). For the purpose of dividing data into two clusters, we consider NCut based graph partitioning. This choice is motivated by the robustness of spectral clustering \cite{ElhamPAMI}. In NCut, each data point in the data matrix  $\mathcal{D}$ is mapped to a vertex of a weighted undirected graph $G=(V,E)$, where the edge weights correspond to the similarity between the two vertices~\cite{ngspectral,von2007tutorial}. The adjacency matrix $A \in \mathds{R}^{n_d \times n_d}$ is  computed as  \begin{equation}\label{eq:adj} A_{i,j}=\begin{cases} &\exp{\bigl(-\frac{1}{2}|x_i-x_j|^\top \Sigma^{-1}|x_i-x_j|\bigr)} \text{ if } i\ne j\\ & 0 \text{ if } i = j,\end{cases} \end{equation}where the parameter $\Sigma$ controls the connectivity. The corresponding degree  matrix is defined as \begin{equation}D(i,j)=\begin{cases}&\sum_{i=1}^{n_d} {A(i,j)}\text{ if } i = j\\& 0 \text{ if } i \ne j.\end{cases}\end{equation} 
If $p_r$ and $p_l$ are the two  partitions, the normalized cut (NCut) objective function
~\cite{fowlkes2004spectral,NormCut} is given by \begin{equation}J_{nc}=\sum_{i\in |p_r|, j \in |p_l|}{A(i,j) }(\frac{1}{v_{r}}+\frac{1}{v_{l}}), \label{J_nc}\end{equation} where $v_{r}$ and $v_{l}$ are the volumes of  $p_r$ and $p_l$, given by the sum of all edge weights attached to the vertices in that partition \begin{equation} v_{r}=\sum\limits_{i\in p_r}{D(i,i)}, v_{l}=\sum\limits_{i\in p_l}{D(i,i)}.\end{equation} 
Shi and Malik~\cite{NormCut} have shown that the NCut objective function is equivalent to \begin{equation}J_{nc}=\frac{y_{r}^\top(D-A)y_{r}}{y_{r}^\top Dy_{r}},\end{equation} where $y_{r}$ is an indicator vector \begin{equation}y_{r}(i)=\begin{cases}&v_l \text{ if } i \in p_r\\&- {v_{r}}{} \text{ if } i \in p_l.\end{cases}\end{equation}

An exhaustive search for a $y_r$, such that $J_{nc}$ is minimized, is NP complete. However, if $y_r$ is relaxed to have real values then an approximate solution can be obtained by the generalized eigenvalue problem \begin{equation}(D-A)\hat{y}_{r}=\lambda D \hat{y}_{r}.\end{equation} Transforming to the standard eigen-system \begin{equation}D^{-\frac{1}{2}}(D-A)D^{-\frac{1}{2}}\hat{q}_{r}=\lambda  \hat{q}_{r},\end{equation} where $\hat{q}_{r}=D^{\frac{1}{2}}\hat{y}_{r}$ is an approximation to $q_r$ and \begin{equation}L_{s}=D^{-\frac{1}{2}}(D-A)D^{-\frac{1}{2}}\end{equation} is a symmetric normalized Laplacian matrix. 
 For a connected graph, only one eigenvalue of $L_s$ is zero which corresponds to the eigenvector given by $\textbf{u}_n=\sqrt{\text{diag}(D)}$ and $\lambda_n=\textbf{u}_n^\top L_s \textbf{u}_n={0}$.

The smallest non-zero eigenvalue $\lambda_{n-1}$ of $L_s$ represents the algebraic connectivity of the graph~\cite{fiedler1973algebraic} and the corresponding eigenvector is known as the Fiedler vector \begin{equation}\textbf{u}_{n-1}=\min_{\textbf{u} \perp \textbf{u}_n} (\textbf{u}^\top L_s \textbf{u}).\end{equation}
The signs of the Fiedler vector can be used to divide the graph into two spectral partitions
 \begin{equation}\ell(x_i)=\begin{cases}& p_r \text{  if  } \textbf{u}_{n-1}(i)\ge0\\&p_l \text{  if  } \textbf{u}_{n-1}(i)<0.\end{cases}\end{equation}

Algorithm \ref{algo:SHC} combines the NCut based graph partitioning with semi-supervised hierarchical clustering. It takes the proximity matrix $A$ defined in \eqref{eq:adj}, class labels $\ell_\mathcal{D}$, cluster labels matrix $\ell_k$, current cluster ID $c$, and the current recursion level $r$. The dummy label given to the test set $\widehat{\ell}_p$ is unique from the class labels. Initially $\ell_\mathcal{k}$ is set to all ones, $c=1$, $r=1$ and at each iteration, the algorithm updates these values.  The algorithm stops when all clusters become {\em indivisible}.

We perform an experiment on synthetic data comprising two random 3D Gaussian clusters with means $\mu_1=[0~ 0~ 0]^\top$ and $\mu_2=[d~ 0~ 0]^\top$, where $d$ is the distance between the cluster centers which is varied from $2$ to $7$ (Fig. \ref{fig:twoClusters}a). Both clusters have the same variance $\sigma I$, where $\sigma=2$ and $I$ is a $3\times 3$ identity matrix. The size of each cluster grows from $100$ to $500$ data points in steps of $25$ points. NCut is applied to the data and the relationship between different parameters is analyzed in Figure \ref{fig:twoClusters}. Notice how the value of $\lambda_{n-1}$ remains fairly stable, given a fixed $d$, for increasing number of data points. This empirically varifies that the value of $\lambda_{n-1}$ is a close approximation of the NCut objective function \eqref{J_nc} \cite{NormCut}. Comparing Fig.~\ref{fig:twoClusters}b and \ref{fig:twoClusters}c, we notice that the magnitude of  $\lambda_{n-1}$ is also proportional to the classification error.  


A sign change in the Fiedler vector $\textbf{u}_{n-1}$  basically means that the corresponding data point has moved from one cluster to the other. Since Fiedler vector is computed iteratively, we can count the  number of sign changes till convergence (Section \ref{sec:DFV}). Hence, we can find the number of points switching partitions and the number of times each point switches partitions. We observe, in Fig. \ref{fig:twoClusters}d, that the average number of sign changes (points switching partitions) reduces as $\lambda_{n-1}$ reduces (i.e. $d$ increases). Figure \ref{fig:twoClusters}e  shows the classification error only for those points that swith partitions 0, 1 and 2 times. We observe  that those data points which change signs (switch partitions) more frequently have higher error rates.  

If a parent cluster is divided into two balanced child clusters, the size of each child's local proximity matrix is four times smaller than the parent's matrix. Since the complexity of eigenvector solvers is $O(n^3)$, the complexity reduces in the next iteration by $O((\frac{n}{2})^3)$ for each sub-problem. The depth of the recursive tree is $\log_2(n_d)$, however, the proposed supervised stopping criterion does not let the iterations to continue until the very end.  The process stops as soon as all clusters are indivisible. A computational complexity analysis of the recursion tree reveals that the overall complexity of the eigenvector computations remains the same, $O(n_d^3)$. 
 
\begin{figure}
\centering
\includegraphics[width=8.5cm]{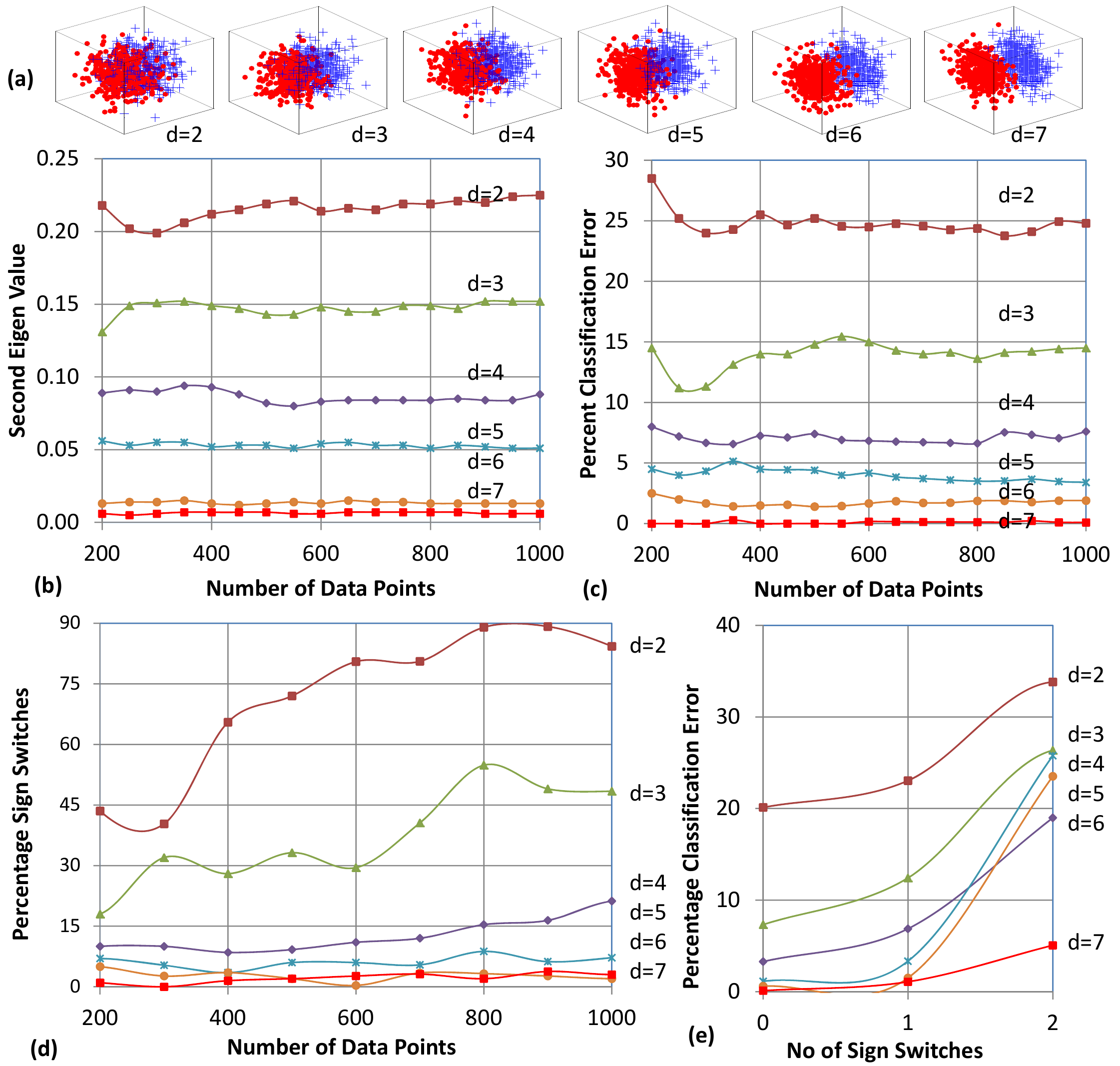}
\caption{(a) 3D Gaussian clusters with center to center distance varying from $d=2$ to $d=7$ units. (b) Variation of algebraic connectivity $\lambda_{n-1}$ with graph size and distance $d$. (c) Classification error versus the graph size and  $d$ (d) Average number of sign switches of Fiedler vector coefficients versus the graph size and $d$  (e) Classification error is more in data points which change signs (switch partitions) more frequently.}
\label{fig:twoClusters} \end{figure}

\subsection{Proximity Matrix Computation through Regularized Linear Regression}
Often high dimensional data sets lie on  low dimensional manifolds. In such cases, the Euclidean distance based adjacency matrix \eqref{eq:adj} may not be an effective way to represent the geometric relationships among the data points. A more viable option is the sparse representation of data which has been used for many tasks including label propagation~\cite{SparseLabel}, dimensionality reduction, image segmentation and face recognition~\cite{Wright_TPAMI_2009}. Recently, Elhamifar and Vidal ~\cite{ElhamPAMI} proposed sparse subspace clustering which can discriminate data lying on independent and disjoint subspaces. 

A vector can only be represented as a linear combination of other vectors spanning the same subspace. Therefore, the proximity matrices based on linear decomposition of data lead to subspace based clustering. 
Representing a data point $x_i$ as a linear combination of the remaining data points $ (\mathcal{\hat{D}})$ ensures that zero coefficients will only correspond the points spanning  subspaces independent to the subspace spanned by $x_i$. Such a decomposition may be computed with least squares: $\alpha_i=(\mathcal{\hat{D}}^\top \mathcal{\hat{D}})^{-1}\mathcal{\hat{D}}^\top x_i$, where $\alpha$ are the linear coefficients. For high dimensional data,  $\mathcal{\hat{D}}^\top \mathcal{\hat{D}}$ can be rank deficient.  Inverse may be computed through eigen decomposition: $D^\top D=USU^\top$, where $U$ are the eigenvectors of $D^\top D$ and $S$ is the diagonal matrix of singular values. By selecting non-zero singular values  and the corresponding eigenvectors, $D^\top D=\hat{U}\hat{S}\hat{U}^\top$ and $\alpha_i=\hat{U}\hat{S}^{-1}\hat{U}^\top\mathcal{\hat{D}}^\top x_i$.

One may introduce sparsity by only considering the $w$ largest coefficients in $\alpha_i$ and forcing the rest to zero. An alternate approach is to use $\ell_1$ regularized linear regression also known as Lasso~\cite{Tibshirani94} 
\begin{equation}\label{eqn:lasso2}\alpha_i^* \defeq \min_{\alpha_i} \left(\frac{1}{2}||x_i-\mathcal{\hat{D}}\alpha_i||^2_2 + w||\alpha_i||_1  \right)~,
\end{equation} where $||\alpha_i||_1$ approximates the sparsity induction term and $w_i>0$ is the relative importance of the sparsity term. In \eqref{eqn:lasso2}, $\ell_1$ is a soft constraint on sparsity and the actual sparsity (number of non-zero coefficients) may vary. To ensure a fixed sparsity linear regression 
\begin{equation}\label{eqn:womp}\alpha_i \defeq \min_{\alpha_i} \left(\frac{1}{2}||x_i-\mathcal{\hat{D}}\alpha_i||^2_2\right) \text{such that} ||\alpha_i||_o \le w, \end{equation} can be used. We observed that while this approach is faster, it yeilds less accuracy compared to the $\ell_1$ formulation.

To complete the proximity matrix $S$, the same process is repeated for all $x_i$ and the corresponding  $\alpha_i$  are appended as columns 
$S=\{\alpha_i\}_{i=1}^{n_d}\in\mathds{R}^{{n_d}\times
{n_d}}$. Some of the $\alpha$ coefficients may be negative and in general $S(i,j) \ne S(j,i)$. Therefore, a symmetric sparse LS proximity matrix is computed as $A=|S|+|S^\top|$ for spectral clustering.

\section{Spectral Clustering on Grassmannian Manifolds}
Eigenvectors computation of large Laplacian matrices $L_{s}\in\mathds{R}^{n_d\times n_ d}$ incurs high computational cost.  A naive approach to reduce the cost is to compute eigenvectors for randomly sampled columns of $L_{s}$  and  extrapolate to the rest of the data~\cite{fowlkes2004spectral,pengscalable,Kwok}.  Instead, we replace each image-set $X_i=\{x_j\}_{j=1}^{n_i} \in\mathds{R}^{l\times n_i}$ by a compact representation and perform clustering  on the  representation. Our choice of compact representation is motivated from linear subspace based image-set representations~\cite{wang2012manifold,Kim_TPAMI_2007}. These subspaces can be considered as points on Grassmannian manifolds~\cite{hamm2008grassmann,harandi2011graph}. While others perform discriminant analysis on Grassmannian manifolds or compute manifold to manifold distances, we propose sparse spectral clustering on Grassmannian manifolds.

A set of $d$-dimensional linear subspaces of $\mathds{R}^{n}$, $n=\min(l,n_i)$  and $d \le n$, is termed the Grassmann manifold $Grass(d, n)$~\cite{grassmann}. An element $\mathcal{Y}$ of $Grass(d, n)$ is a $d$-dimensional subspace which can be specified by a set of $d$ vectors: $Y=\{y_1, . . . , y_d\}\in \mathds{R}^{l\times d}$ and $\mathcal{Y}$ is the set of all linear combinations. For each data element of the image-set, we compute a set of basis ${Y}$ and the set of all such ${Y}$ matrices is termed as a non-compact Stiefel manifold  $ST(d,n)\defeq \{{Y} \in R^{l \times d}: \text{ rank }(\mathcal{Y} ) = d\}$.  We arrange all the $Y$ matrices in a basis matrix $B$ which is capable of representing each data point in the image-set by using only $d$ of its columns. For the $i^{th}$ data point in the $j^{th}$ image-set $x_j^i \in X_j$, having $B_j$ as the basis matrix, $x_j^i=B_j\alpha_j^i$, where $\alpha_j^i$ is the set of linear parameters with $|\alpha_j^i|_o=d$. For the case of known $B_j$,  we can find a matrix $\alpha_j =\{\alpha_j^1, \alpha_j^2, \cdots \alpha_j^{n_i}\}$  such that  the residue is minimized  \begin{equation}\label{eqn:lasso}  \min_{\alpha_j}(\sum\limits_{i=1}^{n_i}||x_j^i-B_j\alpha_i^j||^2_2) \text{ ~~s.t. } ||\alpha_j||_o \le d ~. \end{equation} Using the fact that $\ell_o$ norm can be approximated by $\ell_1$ norm, we can estimate both $\alpha_j$ and $B_j$ iteratively by using the following objective function~\cite{Lee07efficientsparse}\begin{equation}\label{eqn:loss}\min_{\alpha_j,B_j}\left(\frac{1}{n_j}\sum_{i=1}^{n_j} {\frac{1}{2}||x_j^i-B_j\alpha_j^i||^2_2 \text{   ~~s.t. }||\alpha_j^i||_1 \le  d}\right)~. \end{equation} The solution is obtained by randomly initializing $B_j$ and computing $\alpha_j$, then fixing $\alpha_j$ and recomputing $B_j$ until convergence. The columns  of $B_j$  are significantly smaller than the  number of data points $n_j$ in the corresponding image-set leading to computational cost reduction. Interestingly, this compact representation also increases the classification accuracy of the proposed CVC algorithm because the underlying subspaces of each class are robustly captured in $B$ while the outliers are discarded.

\subsection{Ensemble of Spectral Classifiers}\label{sec:ensemble} Representing  image-sets with Grassmannian manifolds facilitates the formation of an ensemble of spectral classifiers. For an image-set $X_i$ varying the random initializations and dimensionality of $B_i$ in \eqref{eqn:loss} will converge to different solutions.  During off line training, we compute a set of manifolds of varying dimensionality for each image-set $X_i \equiv$ \{$B_i^1,B_i^2,...B_i^\kappa$\}. Spectral clustering is independently performed on each set of manifolds of the same dimensionality  $\{B_1^j,B_2^j,...B_{g}^j,B_{g+1}^j\}$, where $B_{g+1}^j$ is the probe set representation.

The distance of each training set is computed from the probe set. For example,  Bhattacharyya distance of image-set $X_i$ with representation $B_j$ from the probe set $X_p$ is given by  $\mathcal{B}_{i,j,p}$ \eqref{bhat}. For the $j^{th}$ manifold representation, we get a set of distances $\mathcal{B}_{j,p}\in \mathds{R}^{g}$. We consider two simple fusion strategies, sum rule and mode or maximum frequency rule. In sum rule, we sum all the distance vectors $\mathcal{B}_p=\mathcal{B}_{1,p}+\mathcal{B}_{2,p}+...\mathcal{B}_{\kappa,p}$ and  the probe set label corresponds to the gallery set with minimum overall distance  $L_p\equiv\min_i(\mathcal{B}_p(i))$.

In mode based fusion, for the $j^{th}$ classifier, probe set label is independently estimated as the label of  $i^{th}$ gallery set with minimum distance: $L_{j,p}^i\equiv\min_i (\mathcal{B}_{j,p}(i))$, and the final label is the mode of all labels: $L_p \equiv$ mode($\{L_{j,p}^i\}_{j=1}^\kappa$). We empirically observe that the mode based fusion is more robust in case of noisy image-sets and often generates more accuracy than the sum rule based fusion. 

\section{Direct Fiedler Vector Computation}
\label{sec:DFV}

Spectral clustering research is based on generic eigensolvers. As discussed in Section \ref{SHC}, the signs of the  eigenvector coefficients  are more important than the numerical accuracy of their magnitudes. Therefore, while iteratively computing eigenvectors, we enforce the convergence of the signs of eigenvector coefficients instead of the eigenvalues. The iterations terminate when the number of sign changes reduces below a threshold.

We consider power iterations, a fundamental algorithm for  eigen computation~\cite{Matrix}. To find an eigenvector of  $L_s$, a random vector $v^{(0)}_1$ is repeatedly multiplied with $L_s$ and normalized until after $k$ iterations $L_sv^{(k)}_1=\lambda_1 v^{(k)}_1$, where $\lambda_1$ is the maximum eigenvalue and  $u_1=v^{(k)}_1$ is the most dominant eigenvector of $L_s$. The convergence of the power iterations algorithm depends on the ratio of first two eigenvalues $\lambda_2/ \lambda_1$. If this ratio is very small, the algorithm will converge in few iterations. Once $u_1=v^{(k)}$ is found, $L_s$ is deflated to find the next dominant eigenvector
 \begin{equation}
 L_s^{(i+1)}=L_s^{(i)} -  u_1 u_1^{\top} L_s^{(i)},
 \end{equation}
and the same process is repeated. Thus the computation of eigenvectors proceeds from the most significant to the least significant without skipping any intermediate vector. Such a computational order is very inefficient from the spectral clustering perspective where only the least significant eigenvectors are required. For the specific case of 2-way hierarchical clustering, only the eigenvector corresponding to the second smallest eigenvalue is required. This is the main reason why hierarchical spectral clustering was not previously considered a viable option \cite{NormCut}. We solve this problem by proposing Algorithm \ref{algo:Fiedler} for the direct computation of the Fiedler vector using an inverse iteration method~\cite{Matrix}. Moreover, since the signs are more important than the numerical values, our proposed algorithm stops when most of the signs of the Fiedler vector stabilize, further reducing the computational cost.

\begin{algorithm}[t]
\caption{Direct Fiedler Vector Computation}
\begin{algorithmic}
 \REQUIRE $L_s$, $u_n$, $\eta$, $\epsilon_s$ 
 \ENSURE $u_{n-1}, \lambda_{n-1}$ \{Fiedler Vector and Value \}
 \STATE $\hat{L}_s\leftarrow L_s-u_n \textbf{1}^{1\times n}-\eta I$ \{Over Deflation and Eigen-shift\}
 \STATE $Q_sR_s \leftarrow \hat{L}_s$ \{QR Decomposition\}
 \STATE $v^{(0)}\leftarrow\sum_{i=1}^{n-1}{q_i}-u_n$, $v^{(0)}\leftarrow{v^{(0)}}/{||v^{(0)}||_2}$ 
 \STATE $\Delta_s\leftarrow \epsilon_s+1$, $k\leftarrow1$
 \WHILE{$\Delta_s \ge \epsilon_s$}
       \STATE $v^{(k')}\leftarrow Q_s^\top v^{(k-1)}$
       \STATE $w\leftarrow\textbf{0}^{n\times 1}$
       \FORALL {$i\leftarrow n:-1:1$}
       \STATE $w(i)\leftarrow(v^{(k')}(i)-R_s(i,:)w)/R(i,i)$
       \ENDFOR
     \STATE $v^{(k)}\leftarrow{w}/{||w||_2}$
     \STATE $\Delta_s \leftarrow \sum{(v^{(k)}>0)\oplus (v^{(k-1)}>0)}$
     \STATE $k\leftarrow k+1$
      \ENDWHILE
         \STATE $u_{n-1}\leftarrow v^{(k)}$, $\lambda_{n-1}\leftarrow u_{n-1}^\top L_s u_{n-1}$ 
 \end{algorithmic}
\label{algo:Fiedler}
\end{algorithm}

Graph Laplacian is a symmetric matrix, therefore its eigen decomposition is $L_s=U^\top \Lambda U$ and $L_s^{-1}=U^\top \Lambda^{-1} U$. For a scalar  $\eta$ and identity matrix $I$ one can show that
\begin{equation}
(L_s-\eta I)^{-1}=U^\top(\Lambda-\eta I)^{-1}U.
\end{equation}
That is, the eigenvectors of $(L_s-\eta I)^{-1}$ are the same as that of $L_s$ while the eigenvalues  are $1/(\lambda_i-\eta)$. If $\eta$ is the same as the k$^{th}$ eigenvalue, then $1/(\lambda_k-\eta)\rightarrow \infty$ and will become significantly larger than all other eigenvalues. If power iteration method is applied to $(L_s-\eta I)^{-1}$, it will converge  in only one iteration. However, if the difference between $\eta$ and $\lambda_k$ is relatively large, power iteration will converge depending on the ratio of $(\lambda_j-\eta)/(\lambda_k-\eta)$, where $\lambda_j$ is  the next closest  eigenvalue. 

To avoid matrix inversion, $(L_s-\eta I)^{-1}$, the matrix vector multiplication step in the power iteration method $$v^{(k)}=(L_s-\eta I)^{-1}v^{(k-1)}$$
is modified as 
\begin{equation}
(L_s-\eta I)v^{(k)}=v^{(k-1)}.
\end{equation} 
To find $v^{(k)}$ we solve a set of linear equations. For this purpose, we use QR decomposition $Q_sR_s=(L_s-\eta I)$, where $Q_s$ is an orthonormal matrix and $R_s$ is an upper triangular matrix. Therefore, the following system of equations
\begin{equation}R_{s} v^{(k)}=Q_{s}^\top v^{(k-1)}\end{equation} is upper triangular and we use back substitution to efficiently solve this system. 

If $\lambda_k$ is known, we can directly compute the $k^{th}$ eigenvector by setting $\eta=\lambda_k$. Otherwise, for a given $\eta$, the algorithm will find an eigenvector with eigenvalue closest to $\eta$. For direct computation of the Fiedler vector, we need an estimate of $\lambda_{n-1}$. Unfortunately, in the spectral clustering literature~\cite{chung1997spectral}, sufficiently tight and easy to compute bounds on $\lambda_{n-1}$ have not been found for generic graphs. The existing bounds are either loose or require significant computational complexity. 

For a connected graph, $\lambda_{n}=0$ and $\lambda_{n-1}>\lambda_n$, for a very small value of $\eta$, the inverse power iteration algorithm may converge to the zero eigenvalue,  while setting a relatively larger value of $\eta$, the algorithm may converge to another eigenvalue $\lambda>\lambda_{n-1}$. We solve this problem by  over-deflating $L_s$ with   $u_n=\sqrt{diag(D)}$, where $D$ is the degree matrix \begin{equation}\hat{L}_s=L_s-u_n 1^{1\times n}.\end{equation} Note that a conventional deflation approach cannot be used because $u_n$ is a null vector  of $L_s$. Over-deflation ensures that the proposed Algorithm 3 does not converge towards $u_n$ no matter how small the value of $\eta>0$ is used. 
   
To find a rough estimate of $\eta$ we use the bound   $ \lambda_{n-1}\ge{1}/{\varphi v_g}$~\cite{chung1997spectral}, 
where $\varphi$ is the diameter of the graph and $v_g$ is the volume of the graph. Note that $v_g$ can be easily computed from the degree matrix while no simple method exists for the computation of $\varphi$, the weight of the longest path in the graph. We use the fact that $\varphi \le v_g$, therefore  $ \lambda_{n-1}\ge{1}/{ v_g^2}$. Which is a loose bound but readily known.   
   
After convergence, Algorithm 3 can find the exact value of $\lambda_{n-1}$ however, we are not interested in the exact solution for the purpose of spectral clustering.  Rather we are  interested in the signs which represent the cluster occupancy.  We compute the cluster indicator vector $q^{k}=v^{(k)}>0$. Instead of setting the convergence criteria on $|\lambda_{n-1}^k-\lambda_{n-1}^{k-1}|$, the iterations are stopped when $\Delta_s=\sum (q^{(k)} \oplus q^{k-1})<\epsilon_s$, where $\oplus$ is XOR operation and $\epsilon_s$ is a user defined threshold. The value of $\Delta_s$ also represents the quality of the NCut. A larger value shows that there are many data points which are switching across the partition in consecutive iterations and a good cut is not found.

\section{Experimental Evaluation}
The proposed algorithms were evaluated for image-set based face recognition, object categorization and gesture recognition.  Comparisons are performed with seven existing state of the art image-set classification algorithms including Discriminant
Canonical Correlation (DCC)~\cite{Kim_TPAMI_2007}, Manifold-Manifold Distance (MMD)~\cite{wang2012manifold}, Manifold
Discriminant Analysis (MDA)~\cite{Wang_CVPR_2009}, linear Affine and Convex
Hull based image-set Distance (AHISD, CHISD)~\cite{Cevikalp_CVPR_2010}, Sparse Approximated Nearest Points~\cite{Hu_TPAMI_2012}, and Covariance Discriminative Learning (CDL)~\cite{CDL}.  The same experimental protocol was replicated for all algorithms. The implementations of~\cite{Kim_TPAMI_2007, wang2012manifold, Cevikalp_CVPR_2010, Hu_TPAMI_2012} were provided by the original authors whereas the implementation of~\cite{Wang_CVPR_2009} was provided by the authors of Hu et al.~\cite{Hu_TPAMI_2012}. We carefully implemented CDL which was verified by the original authors of the algorithm~\cite{CDL}. The parameters of all methods were carefully optimized to maximize their accuracy. For DCC, the embedding space dimension was set to 100, the subspace dimensionality was set to 10 and set similarity was computed from the 10 maximum correlations. For MMD and MDA, the parameters were selected as suggested in \cite{Wang_CVPR_2008} and~\cite{Wang_CVPR_2009}.

Each image-set was represented by $\nu$ manifolds with dimensionality increasing from 1 to $\nu$. Each set of equal dimensionality manifolds was used to make an independent classifier. The results of $\nu$ classifiers were fused using sum and mode rules. Table \ref{tab:CVC} shows the results of the proposed CVC algorithm  using Semi-supervised Hierarchical Clustering (SHC) for the two fusion schemes over all datasets. The proximity matrix was based on the $\ell_1$ norm regularized linear regression \eqref{eqn:lasso2} and solved using the SPAMS~\cite{Mairal1} library with $w=0.01$~\cite{MyBMVC12}. For eigenvector computation, the proposed DFVC Algorithm~\ref{algo:Fiedler} was used.  DFVC was taken to be converged when the sign changes reduced to $ \le 1.00$\%. The code, manifold basis and clustering results will be made available at \url{http://www.csse.uwa.edu.au/~arifm/CVC.htm}.
\begin{table*}
  \centering
 \caption{Accuracy (\%) of CVC algorithm  using Mode fusion, Sum fusion and No-Fusion (NFus) schemes.}
    \begin{tabular}{lrrrrrrrrrrrrrrr}
    \toprule
     & \multicolumn{3}{c}{YouTube~\cite{Kim_CVPR_2008}} & \multicolumn{3}{c}{Mobo~\cite{Gross_TR_2001}} & \multicolumn{3}{c}{HONDA~\cite{Lee_CVPR_2003} }  & \multicolumn{3}{c}{ETH 80~\cite{ETH80} } & \multicolumn{3}{c}{Cambridge~\cite{Camdatabase}}\\
                & Mod & Sum    &NFus  & Mod & Sum     &NFus & Mod& Sum    &NFus  & Mod& Sum    &NFus  & Mod & Sum    &NFus  \\ 
        \cmidrule(r{5pt}l{5pt}){2-4}  \cmidrule(r{5pt}l{5pt}){5-7}   \cmidrule(r{5pt}l{5pt}){8-10} \cmidrule(r{5pt}l{5pt}){11-13} \cmidrule(r{5pt}l{5pt}){14-16}  
    Avg   & 76.03 & 70.28 & 62.27 & 98.33 & 96.39 & 96.53 & 100   & 95.13 & 94.36 & 91.75 & 88.75 & 84.75 & -     & -     & - \\

    Min   & 70.56 & 64.89 & 53.19 & 97.22 & 93.06 & 94.44 & 100   & 89.74 & 89.74 & 87.5  & 75.0    & 77.5  & -     & -     & - \\
    Max   & 81.91 & 76.59 & 66.31 & 100.0   & 98.61 & 100   & 100   & 100   & 100   & 100   & 97.50  & 92.50  & 83.10 & 76.53& 74.58 \\
    STD   & 4.69 & 5.38 & 5.33 & 0.88 & 2.09  & 1.78  & 0.00     & 2.91  & 3.38  & 4.42  & 6.80   & 5.33  & -     & -     & - \\

    \bottomrule
    \end{tabular}
    
  \label{tab:CVC}
\end{table*}%
\begin{figure}
\centering
\includegraphics[width=8.5cm]{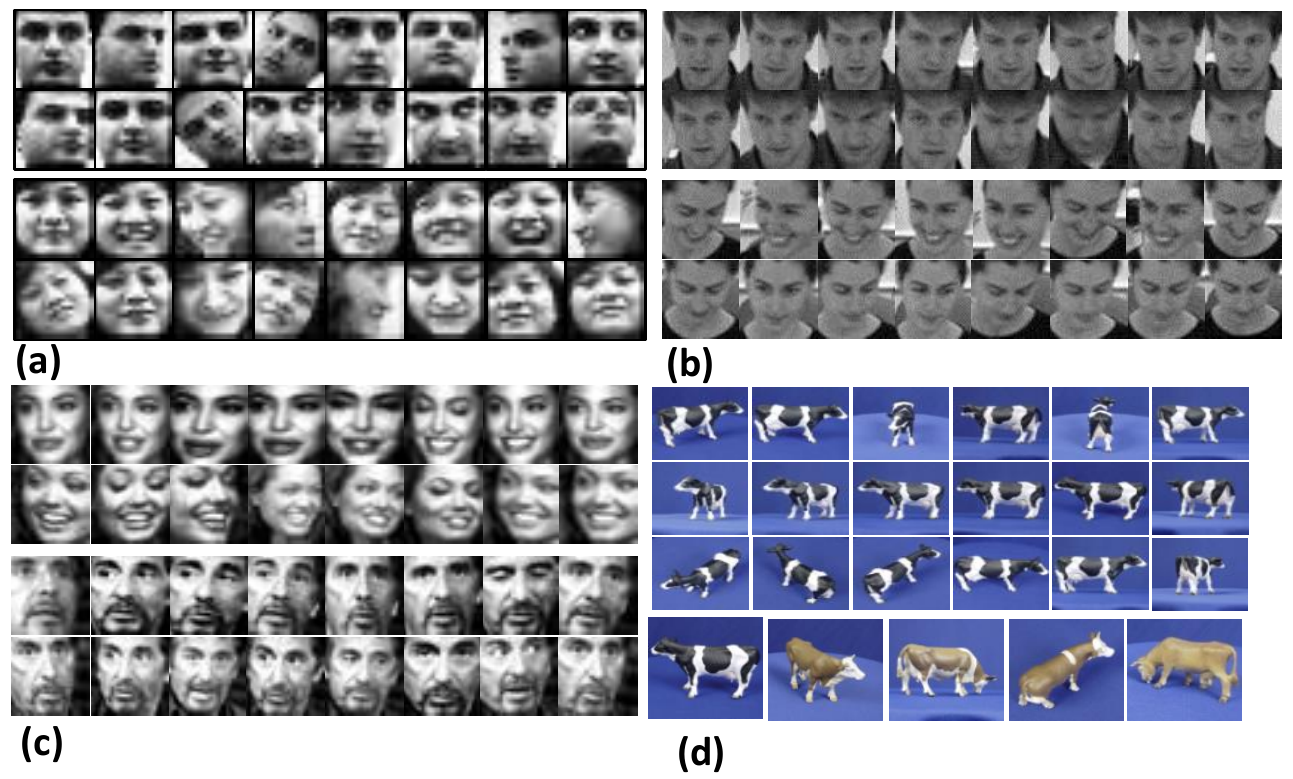}
\vspace{-6mm}
\caption{Two example image-sets from (a) Honda, (b) CMU Mobo
and (c) You-tube Celebrities datasets. (d) One object category from the ETH 80 dataset.}
\label{fig:dataset} 
\vspace{-3mm}
\end{figure}
\subsection{Face Recognition using Image-sets}
Our first dataset is the You-tube Celebrities~\cite{Kim_CVPR_2008} which is very challenging and includes 1910 very low resolution videos (of 47 subjects) containing motion blur, high compression, pose and expression variations (Fig.~\ref{fig:dataset}c). Faces were automatically detected, tracked  and cropped to 30$\times$30 gray-scale images. Due to tracking failures our sets contained fewer images (8 to 400 per set) than the total number of video frames. Experiments were performed on both intensity based features and HOG features. The proposed algorithm performed better on the HOG features. Five-fold cross validation experiments were performed where 3 image-sets were selected for training and the remaining 6 for testing.  

Each image-set was represented by two Grassmannian manifold-sets for $\lambda=\{1, 2\}$ in \eqref{eqn:loss}. Each set had 8 classifiers with dimensionality increasing from 1 to 8.  A different classifier was made for each dimension. Fig. \ref{fig:YoutubeMobo}a shows that the accuracy increases as the dimensionality increases from 1 to 8 for different fusion options.  The average accuracies obtained by mode fusion, sum fusion and no-fusion were \{76.03$\pm$4.69, 70.28$\pm$5.38, 62.27$\pm$5.33\}. Due to noise in the dataset  the mode fusion performed better than the sum fusion. The mode fusion is more robust due to independent decisions at each level while the  errors may get accumulated in sum fusion. Both mode and sum fusions outperformed the no-fusion case. This shows that fusion can be used to compensate for the cases where good quality cuts are not found for graph partitioning. The high dimensional manifolds are more separable in some dimensions and less in the others. The ensemble exploits the more separable dimensions to obtain better accuracy. On this dataset Bhattacharyya and Hellinger distance measures achieved very similar accuracies.  Table \ref{tab:Comp} shows a comparison with the existing state of the art image-set classification algorithms. Note that CVC algorithm with all combinations is more accurate than the previous best reported accuracy. The average accuracy of CVC algorithm is 76.03$\pm$4.69\% outperforming the existing methods by 11.03\%.  
 
\begin{figure}
\centering
\includegraphics[width=8.6cm]{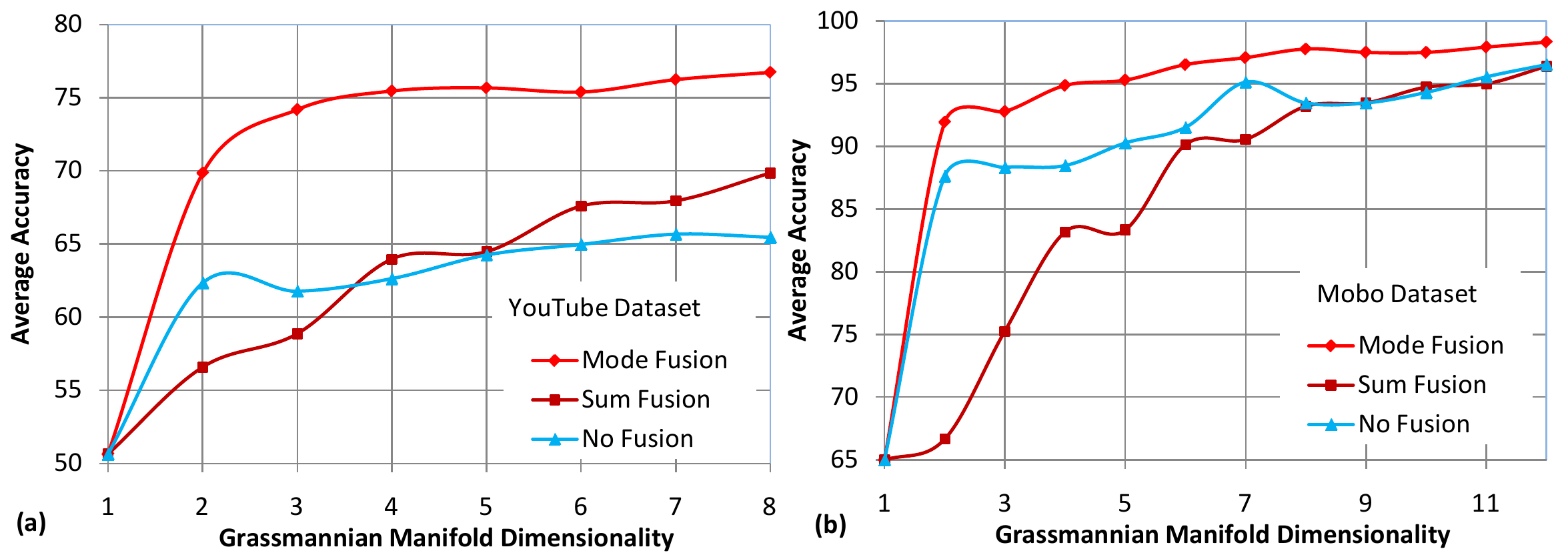}
\caption{CVC Algorithm performance on (a) YouTube and (b) CMU Mobo datasets.}
\label{fig:YoutubeMobo} 
\vspace{-5mm}
\end{figure}

\begin{figure}
\centering
\includegraphics[width=8.6cm]{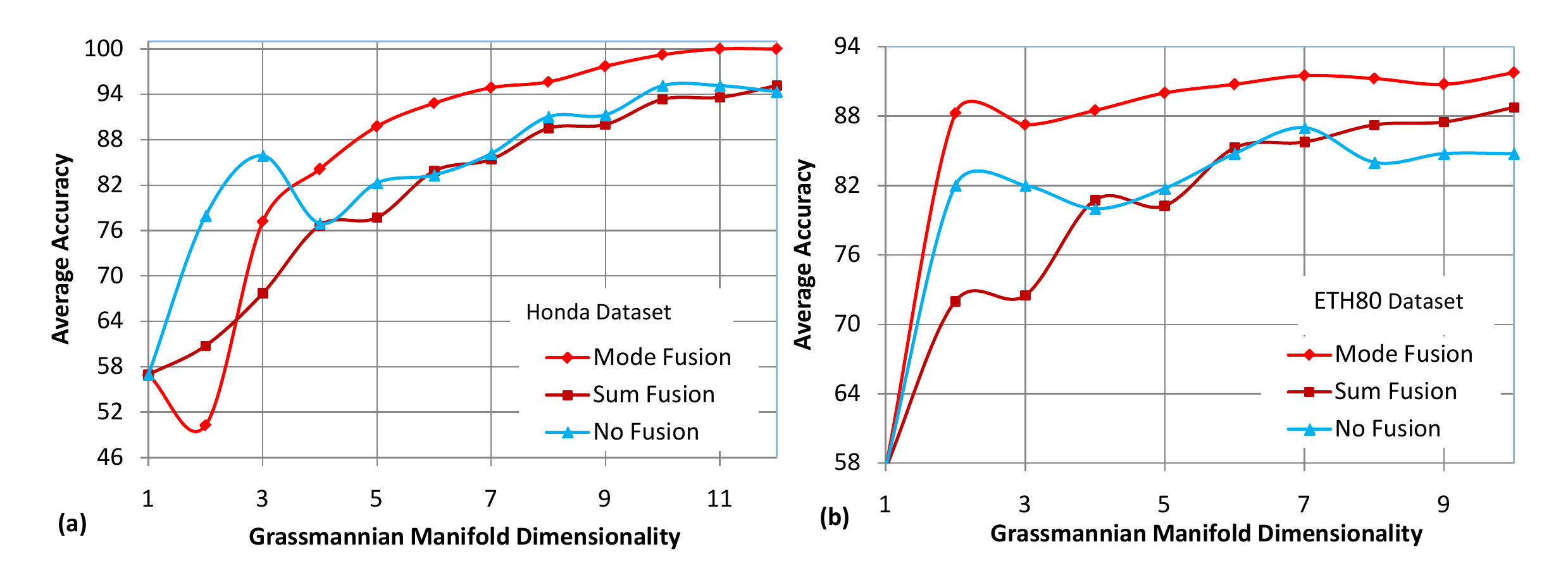}
\caption{CVC algorithm performance on (a) Honda and (b) ETH80 datasets. }
\label{fig:MoboG} 
\end{figure}

The second dataset we used is CMU Mobo~\cite{Gross_TR_2001} containing 96 videos of 24 subjects. Face images were resized to $40\times 40$ and LBP features were computed using circular (8, 1) neighborhoods extracted from $8\times 8$ gray scale patches similar to \cite{Cevikalp_CVPR_2010}. We
performed 10-fold experiments by randomly selecting one image-set per subject as training and the
remaining 3 as probes. This experiment was repeated by varying the manifold dimensionality from 1 to 12. The number of classifiers in the ensemble also vary from 1 to 12. Fig. \ref{fig:YoutubeMobo}b shows the accuracy of CVC algorithm versus the number of classifiers.  Most of the accuracy is gained by the first eight classifiers and any further increase provides minor improvement.  

For the Mobo dataset, the mode fusion again performed better than the sum fusion (Table \ref{tab:CVC}). The improvement over sum rule is relatively smaller compared to the Youtube dataset because Mobo dataset is less noisy. Thus, the no-fusion case also has very similar average accuracy to the sum fusion but with a higher standard deviation.  On this dataset, CVC algorithm achieved a maximum of 100\% and average 98.33$\pm$0.88\% accuracy which is again the highest  (Table \ref{tab:Comp}).

Our final face recognition dataset is Honda/UCSD~\cite{Lee_CVPR_2003} containing 59 videos of 20 subjects with varying poses
and expressions. Histogram equalized 20$\times$20 gray scale face image pixel values were used as features~\cite{wang2012manifold}. We performed 10-fold experiments by randomly selecting one set per subject as gallery and the remaining 39 as probes. In this dataset, each image-set was represented by 12 manifolds of dimensionality 1 to 12. The  CVC algorithm achieved 100\% accuracy with mode fusion (Table \ref{tab:CVC}). Note that the accuracies of other algorithms in Table \ref{tab:CVC} for the Honda dataset are different than those reported by their original authors because they are either used a single fold \cite{Hu_TPAMI_2012} or different folds.

\begin{table}[t]
\caption{Average recognition rates comparisons for 10-fold experiments on Honda, Mobo and ETH80, 5-folds on You-Tube and 1-fold on Cambridge dataset.}
\vspace{-6mm}
\begin{center}{
\resizebox{8.5cm}{!}{
\begin{tabular}{lccccc}
\toprule
        &You-Tube  & MoBo &Honda & ETH80  & Cambr. \\
\midrule
    
    DCC   & 53.9$\pm$4.7 & 93.6$\pm$1.8 & 94.7$\pm$1.3 & 90.9$\pm$5.3 & 65 \\

    MMD   & 54.0$\pm$3.7 & 93.2$\pm$1.7 & 94.9$\pm$1.2 & 85.7$\pm$8.3 & 58.1 \\
          
    MDA   & 55.1$\pm$4.5 & 97.1$\pm$1.0 & 97.4$\pm$0.9 & 80.5$\pm$6.8 & 20.9 \\
          
    AHISD & 60.7$\pm$5.2 & 97.4$\pm$0.8 & $\ddagger$89.7$\pm$1.9 & 74.76$\pm$3.3 & 18.1 \\
          
    CHISD & 60.4$\pm$5.9 & 96.4$\pm$1.0 & $\ddagger$92.3$\pm$2.1 & 71.0$\pm$3.9 & 18.3 \\
         
    SANP  & 65.0$\pm$5.5 & 96.9$\pm$0.6 & 93.1$\pm$3.4 & 72.4$\pm$5.0 & 22.5 \\
    CDL   & *62.2$\pm$5.1 & 95.8$\pm$2.0 & \textbf{100$\pm$0.0} & 89.2$\pm$6.8 & 73.4 \\
    Prop. & \textbf{76.03$\pm$4.69} & \textbf{98.3$\pm$0.9} & \textbf{100$\pm$0.0} & \textbf{91.5$\pm$4.4} & \textbf{83.1} \\

\bottomrule
\end{tabular}}

}\end{center}
\label{tab:Comp}
\small{~~~~* CDL results are on different folds therefore, the accuracy is less than that reported by~\cite{CDL}. $\ddagger$ The accuracy of AHISD and CHISD is less than that reported in ~\cite{Cevikalp_CVPR_2010} due to smaller image sizes.} 
\end{table}

\subsection{Object Categorization \& Gesture Recognition}
For object categorization, we used the ETH-80 dataset~\cite{ETH80} containing images of 8 object categories each with 10 different objects. Each object has 41 images taken at different views forming an image-set. We used 20$\times$20 intensity images for classifying an image-set of an object into a category. ETH-80 is a challenging database because it has fewer images per set and significant appearance variations across objects of the same class. Experiments were repeated 10 folds. Each time for each class, 5 random image-sets were used  for training and the remaining 5 for testing. Each image-set was represented with manifolds of dimensionality varying from 1 to 10. 

A comparison of mode, sum and no-fusion schemes is shown in Fig.  \ref{fig:MoboG}b. As the dimensionality of the manifold increases, the accuracy of all fusion schemes increases. The mode fusion decreases at $\nu=3$ showing that the third dimension of the manifold was less discriminative. Despite  variations, mode fusion obtained better accuracy for $\nu>1$. Table \ref{tab:Comp} shows that the CVC algorithm outperformed all other methods on the ETH-80 dataset giving a significant margin from algorithms that performed quite well on the face datasets such as SANP.

\begin{figure}
\centering
\includegraphics[width=8.6cm]{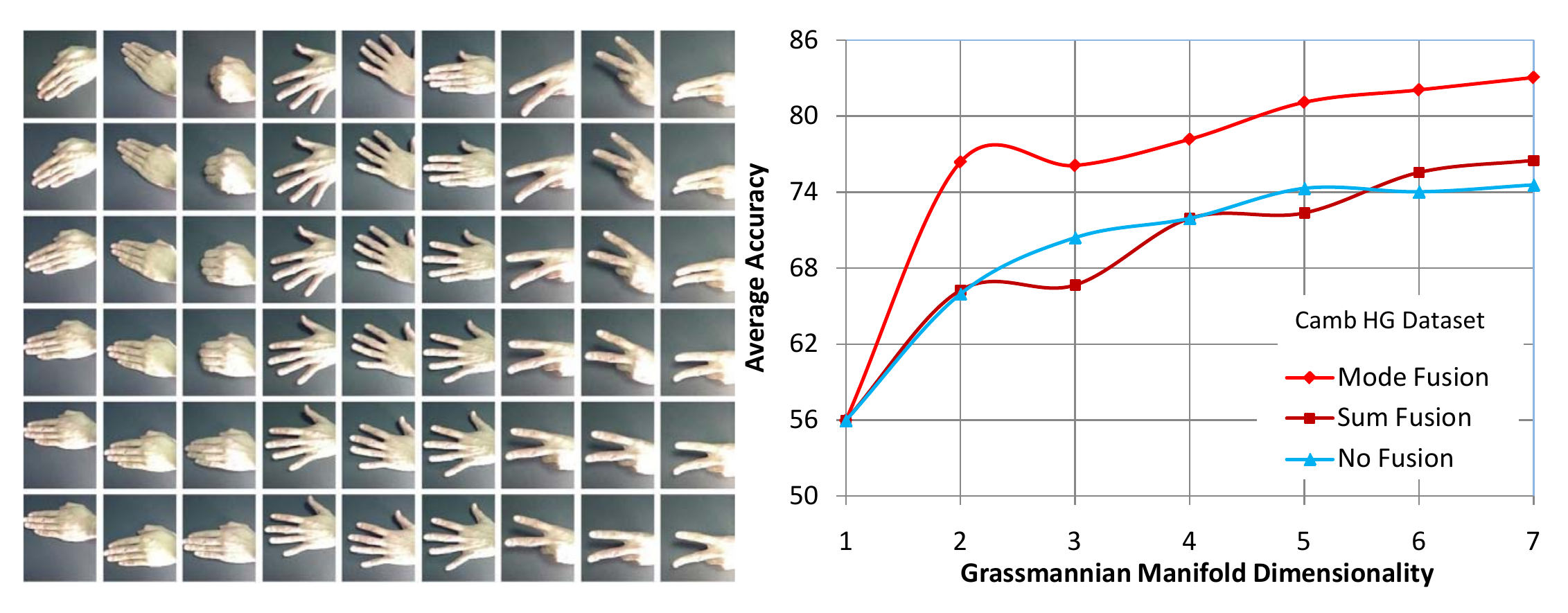}
\caption{(a) Cambridge Hand Gesture dataset. (b) CVC algorithm performance on Cambridge HG dataset.  }
\label{fig:ETHG} 
\vspace{-1mm}
\end{figure}

\begin{table}
  \centering
  \caption{Comparison of CVC algorithm (no-fusion) accuracy by using SHC,  K-Means, SSC~\cite{ElhamPAMI} and SKMS~\cite{Anand}.}
    \begin{tabular}{rcccc}
    \toprule
          & SHC & K-means & SSC   & SKMS \\
    \midrule
    Youtube & 68.44 & 40.85 & 63.83 & - \\ 
    Mobo  & 94.44 & 61.53 & 93.60  & 72.23 \\
    Honda & 97.44   & 29.50  & 88.33 & 69.23 \\
    ETH80 & 90.00  & 73.00    & 87.50 & - \\ 
    Cambridge HG & 74.58  & 73.15  & 79.58 & - \\ 
    \bottomrule
    \end{tabular}%
  \label{tab:cmpSHC}%
\end{table}%

\begin{table}[htbp]
  \centering
  \caption{ Execution time (sec) comparison of SHC algorithm with K-means, SSC and SKMS clustering algorithms when matching a single probe set to the gallery.}
    \begin{tabular}{rrrrr}
    \toprule   
    & SHC   & K-means & SSC   & SKMS \\
           \midrule
    Youtube & 7.30   & 3.17  & 20.57 & 109.34 \\
    Mobo  & 0.43  & 0.24 & 1.08 & 2.97 \\
    Honda & 0.41  & 0.19  & 0.83 & 1.95 \\
    ETH80 & 1.21 & 1.71 & 3.03 & 10.63 \\
    Cambridge HG & 22.75  & 85.72  & 160.00 & 1069.00 \\
    \bottomrule
    \end{tabular}%
  \label{tab:time}%
\end{table}%

Our last dataset is the Cambridge Hand Gesture dataset \cite{Camdatabase}  which contains 900 image-sets of 9 gesture classes with large intra-class variations (Fig. \ref{fig:ETHG}a). Each class has 100 image-sets, divided into two parts, 81-100  were used as gallery and  1-80 as probes~\cite{KimTensor}.  Pixel values of $20 \times 20$ gray scale images were used as feature vectors. Each image-set was represented by a set of manifolds with dimensionality increasing from 1 to 7. 

The accuracy of CVC algorithm using different fusion schemes versus the dimensionality of Grassmannian manifolds is shown in Fig. \ref{fig:ETHG}b. The accuracy of all fusion schemes increased with increasing the number of classifiers. The maximum accuracy of the proposed CVC algorithm was 83.1\% obtained by  mode fusion (Table \ref{tab:CVC}). Table \ref{tab:Comp}) shows that the CVC algorithm outperformed the state of the art image-set classification algorithms by a significant margin. Notice that some of the image-set classification algorithms did not generalize to the hand gesture recognition problem. On the other hand, the proposed algorithm consistantly gives good performance for face recognition, object categorization and gesture recognition.

\subsection{Comparison of CVC when combined with Different Clustering Algorithms}
We compare the performance of CVC algorithm by combining it with different existing clustering algorithms including k-means clustering, Sparse Subspace Clustering (SSC)~\cite{ElhamifarV09} and Semi-supervised Kernel Mean Shift (SKMS) clustering~\cite{Anand} (Table \ref{tab:cmpSHC}). Comparison was performed in the no-fusion setting using Bhattacharyya measure.  For each dataset, a single classifier based on the maximum manifold dimensionality was used. For the case of k-means and SSC the number of clusters were the same as the number of classes in the gallery. For SKMS, the algorithm automatically finds the optimal number of clusters.

For the SSC and the SKMS algorithms, the implementations of the original authors were used.  The proximity matrices proposed by SSC~\cite{ElhamPAMI} was based on the $\ell_1$ norm regularized linear regression \eqref{eqn:lasso2} which was computed by using ADMM~\cite{ADMM} with default parameters recommended by the original authors~\cite{ElhamPAMI}. The SHC proximity matrix was computed with the SPAMS~\cite{Mairal1} library using $w=.01$, Mode=2. The SHC proximity matrix computaion time was significantly less than that required by SSC (Table \ref{tab:time}).

The number of constraints in SKMS were repeated as \{5\%, 10\%, 20\%, 50\%, 100\%\} of the data points in the gallery. For a randomly selected gallery  point, another  random gallery point was  selected and a constraint was defined to show if both points belong to the same class (and hence the same cluster) or not. No constraint was specified for the data points in the test set. For other parameters of the SKMS algorithm, default values were used. We report the best performance of SKMS over these constraints. 

In the SKMS algorithm, the constraints become part of the objective function.  This type of  supervision causes lack of generalization and hence the accuracy reduced. The supervision used in the proposed SHC algorithm keeps the partitioning itself unsupervised while using the labels to decide whether partitioning is required or not. Therefore, the generalization of the SHC algorithm is equivalent to that of the unsupervised clustering algorithms. 

In Table \ref{tab:cmpSHC}, CVC combined with SHC refers to the proposed algorithm wich outperforms others on the first four datasets and has the second best performance on the last dataset.  This  is  because   the class-cluster distributions obtained by the proposed SHC algorithm are conditionally orthogonal reducing the within Gallery similarities and increasing the discrimination across the gallery classes.  CVC algorithm combined with SSC clustering has the highest accuracy on Cambridge dataset and second highest on the remaining datasets.  

CVC combined with SKMS does not perform as good as SHC or SSC because SKMS incorporates the class labels into the objective function. While SKMS may achieve state of the art results for clustering, it is not suitable for Classification Via Clustering (CVC) as it does not expoit the main strength of CVC i.e. to find the label independant internal data structure. Note that we do not report the performance of CVC combined with SKMS on three datasets as these datasets are large and SKMS is computationally demanding (see Table \ref{tab:time}) and we could not fine tune the SKMS parameters.

The performance of CVC combined with k-means is reasonable given the simplicity and fast execution time of k-means (see next section). These results show that CVC is a generic classification algorithm and has good generalization capability when combined with off-the-shelf clustering algorithms especially when the clustering algorithms do not use labels in their objective function.

\subsection{Execution Time Comparisons}

Table \ref{tab:time} compares the execution time of the proposed clustering algorithm SHC with the existing algorithms including k-means, SSC and SKMS on all five datasets. The experiment was performed in no-fusion mode on the maximum dimensionality of the manifold for each dataset. On the average the SHC algorithm is \{1.33, 3.38, 16.47\} times faster than the k-means, SSC and SKMS respectively. Maximum speedups are \{3.77, 7.03, 47.0\}  and the minimum speedups are \{0.434, 2.03, 4.74\} respectively. Maximum speedup is obtained for the Cambridge dataset. The computational advantage of SHC increases with the gallery size.  For SSC, the  proximity matrix computation is significantly slower than that of the proposed SHC algorithm. In SSC, once the proximity matrix is computed, then k-means clustering is used to cluster the rows of the eigenvectors into $n_k$ clusters. In contrast, in the proposed SHC algorithm, the Fiedler vector computation and partitioning are done alternatively. The partitioning step is a binary decision  using  the signs of the Fiedler vector coefficients. The computational cost of the proposed DFVC algorithm reduces if a cut with low cost is present because the algorithm converges very quickly.

The computational cost of the proposed SHC algorithm increases as the dimensionality of the manifold increases. For mode fusion, a faster CVC algorithm can be designed by terminating computations once a decisive number of label agreements are obtained. For example, if  there are $\nu$ classifiers corresponding to 1 to $\nu$ dimensions, then starting from the lowest dimensional classifier, if floor($\nu/2$)+1 classifiers get the same label, the remaining classifiers cannot change the mode and hence the classification decision. Therefore, computations of the remaining classifiers can be safely skipped making the algorithm significantly faster.    

\subsection{Robustness to Outliers}
We performed robustness experiments in a setting similar to~\cite{Cevikalp_CVPR_2010}. The Honda dataset was modified to have 100 randomly selected images per set. In the first experiment, each {\em gallery} set was corrupted by adding 1 to 3 random images from each other gallery set resulting in 19\%, 38\% and 57\% outliers respectively. The proposed CVC algorithm with Semi-supervised Hierarchical Clustering (SHC) achieved 100\% accuracy for all three cases. In the second experiment, the {\em probe} set was corrupted by adding 1 to 3 random images from each gallery set. In this case, the CVC algorithm achieved \{100\%, 100\%, 97.43\%\} recognition rates respectively. The proposed CVC algorithm outperformed all seven algorithms. Fig.~\ref{fig:robustness} compares our algorithm to the nearest  two competitors in both experiments which are~CDL~\cite{CDL} and SANP \cite{Hu_TPAMI_2012}.
\begin{figure}
\centering
\includegraphics[width=8.6cm]{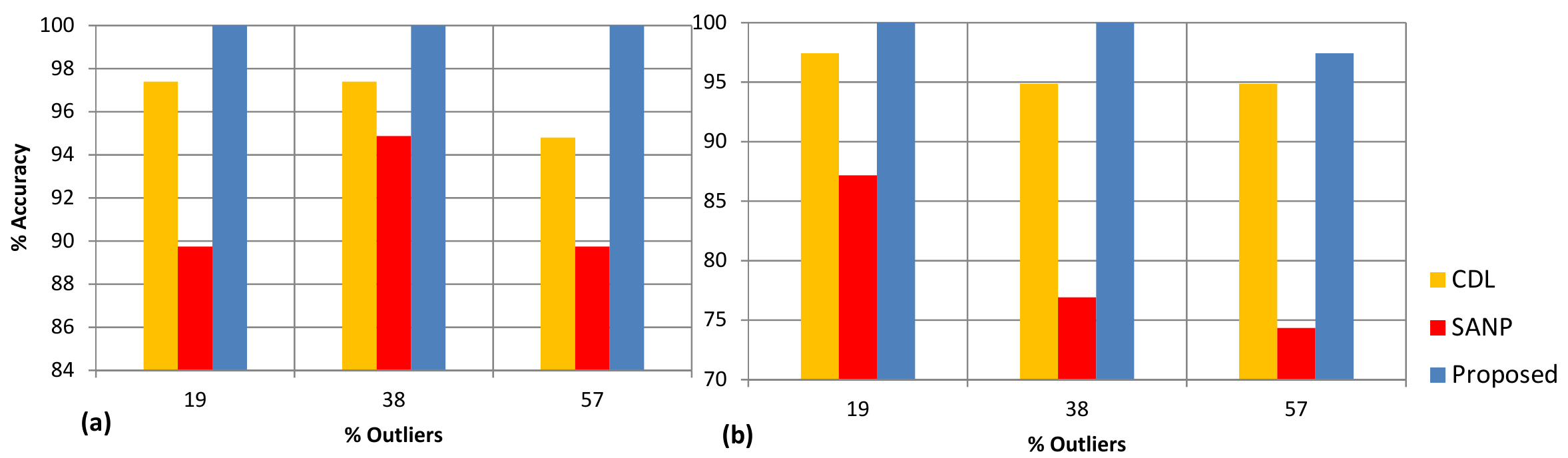}
\vspace{-6mm}
\caption{Robustness to outliers experiment: (a) Corrupted gallery case (b) Corrupted probe case. }
\label{fig:robustness} 
\vspace{-4mm}
\end{figure}

\section{Conclusion}
We presented a Classification Via Clustering (CVC) algorithm which bridges the gap between clustering algorithms and classification problems. The CVC algorithm performs classification by using unsupervised or semi-supervised clustering. Distribution based distance  measures are used to match the class-cluster distributions of the test set and the gallery classes.  A Semi-supervised Hierarchical Clustering (SHC) algorithm is proposed which optimizes the number of clusters using the class labels. An algorithm for Direct Fiedler Vector Computation (DFVC) is proposed to directly compute the second least-significant eigenvector of the Laplacian matrix.  Image-sets are mapped on Grassmannian manifolds and clustering is performed on the manifold bases. By using multiple representations of each set, multiple classifiers are designed and results are combined by mode fusion.  The proposed algorithm consistently showed better performance when compared with state of the art image-set classification algorithms on five standard datasets.

\section{Acknowledgements}

 This research was supported by ARC Discovery Grants DP1096801 and DP110102399. We thank T.~Kim and R.~Wang for sharing the implementations of DCC and MMD and the cropped faces of Honda data. We thank H.~Cevikalp for providing the LBP features of Mobo data and Y.~Hu for the MDA implementation. 



\end{document}